\pdfoutput=1

\documentclass[a4paper]{article}
\usepackage[utf8]{inputenc}

\usepackage[maxbibnames=99,style=alphabetic]{biblatex}
\usepackage{fullpage}

\usepackage{amsfonts, amsmath, amsthm, amssymb}
\usepackage{mathtools}
\usepackage{commath}
\usepackage[bookmarksnumbered,linktocpage,hypertexnames=false,colorlinks=true,linkcolor=blue,urlcolor=blue,citecolor=blue,anchorcolor=green,breaklinks=true,pdfusetitle]{hyperref}
\usepackage{xcolor}
\usepackage{braket}
\newcommand\extrafootertext[1]{%
    \bgroup
    \renewcommand\thefootnote{\fnsymbol{footnote}}%
    \renewcommand\thempfootnote{\fnsymbol{mpfootnote}}%
    \footnotetext[0]{#1}%
    \egroup
}

\usepackage[bb=dsserif]{mathalpha}
\usepackage{bm}

\usepackage[longend,ruled,linesnumbered,nokwfunc]{algorithm2e}
\SetKwInput{Input}{Input}
\SetKwInput{Output}{Output}
\SetKwInput{Guarantee}{Guarantee}
\SetKw{True}{True}
\SetKw{False}{False}
\SetKwComment{Comment}{/* }{ */}

\usepackage[capitalize]{cleveref}
\usepackage{thm-restate}
\newtheorem{theorem}{Theorem}[section]\crefname{theorem}{Theorem}{Theorems}
\newtheorem*{theorem*}{Theorem}\crefname{theorem}{Theorem}{Theorems}
\newtheorem{lemma}[theorem]{Lemma}\crefname{lemma}{Lemma}{Lemmas}
\crefname{claim}{Claim}{Claims}
\crefname{fact}{Fact}{Facts}
\newtheorem{definition}[theorem]{Definition}\crefname{definition}{Definition}{Definitions}
\newtheorem{proposition}[theorem]{Proposition}\crefname{proposition}{Proposition}{Propositions}
\crefname{question}{Question}{Questions}
\crefname{problem}{Problem}{Problems}
\crefname{remark}{Remark}{Remarks}
\crefname{corollary}{Corollary}{Corollaries}

\numberwithin{equation}{section}

\newcommand{\N}{\mathbb N}
\newcommand{\R}{\mathbb{R}}

\newcommand{\cH}{\mathcal{H}} % the Hamiltonian
\newcommand{\cN}{\mathcal{N}} % normal distribution
\newcommand{\cT}{\mathcal{T}} % the discrete set of integration times T.

\newcommand{\Exp}{\mathbb{E}}

\renewcommand{\P}{\mathbb{P}}
\newcommand{\1}{\mathbb{1}}

\DeclareMathOperator{\poly}{poly}
\DeclareMathOperator{\diag}{diag}

\newcommand{\LF}{\mathrm{leapfrog}}

\newcommand{\TV}{\mathrm{TV}}

\renewcommand{\vec}[1]{\bm{#1}}

\newcommand{\dx}{\dif x}
\newcommand{\eps}{\varepsilon}
\newcommand{\tO}{\widetilde{O}}
\newcommand{\tOm}{\widetilde{\Omega}}
\newcommand{\im}{\mathbf{i}\,}

\begin{document}

\title{Hamiltonian Monte Carlo for efficient Gaussian sampling:\\ long and random steps}
\author{Simon Apers\thanks{Universit\'e Paris Cit\'e, CNRS, IRIF, F-75013, Paris, France.
Email: \href{mailto:apers@irif.fr}{apers@irif.fr}, \href{mailto:gribling@irif.fr}{gribling@irif.fr}, \href{mailto:szilagyi.d@gmail.com}{szilagyi.d@gmail.com}
% This work (DS) was partly supported by IdEx Université de Paris ANR-18-IDEX-0001.
} \and Sander Gribling\footnotemark[1] \and D\'aniel Szil\'agyi\footnotemark[1]}
\date{}%\today}

\maketitle
\extrafootertext{\llap{\textsuperscript{$\dag$}}We thank Alain Durmus for useful discussions.}

\begin{abstract}
Hamiltonian Monte Carlo (HMC) is a Markov chain algorithm for sampling from a high-dimensional distribution with density $e^{-f(x)}$, given access to the gradient of $f$.
A particular case of interest is that of a $d$-dimensional Gaussian distribution with covariance matrix $\Sigma$, in which case $f(x) = x^\top \Sigma^{-1} x$.
We show that HMC can sample from a distribution that is $\eps$-close in total variation distance using $\widetilde{O}(\sqrt{\kappa} d^{1/4} \log(1/\eps))$ gradient queries, where $\kappa$ is the condition number of $\Sigma$.

Our algorithm uses \emph{long} and \emph{random} integration times for the Hamiltonian dynamics.
This contrasts with (and was motivated by) recent results that give an $\widetilde\Omega(\kappa d^{1/2})$ query lower bound for HMC with fixed integration times, even for the Gaussian case.
\end{abstract}

\section{Introduction and main result}

One of the most important tasks in statistics and machine learning is to sample from high-dimensional and potentially complicated distributions.
Markov chains are an efficient means for sampling from such distributions, and there is a wide variety of Markov chain algorithms designed specifically for this purpose.
Typically, the main difficulty in analyzing these algorithms is to bound the precise running time or \emph{mixing time} of the Markov chain.
While many algorithms have been in very broad (heuristic) usage for several decades, rigorous bounds on their performance are often missing.
A key example is the \emph{Hamiltonian Monte Carlo} (HMC) algorithm \cite{Duane1987HybridMonteCarlo}.
This is an elegant Markov chain algorithm that utilizes Hamiltonian dynamics to efficiently explore the state space, without straying too far away from the high probability region.
One of its key features is that it overcomes the slow, diffusive behavior that is inherent to ``small step'' approaches such as the ball walk and Langevin algorithm.
While this is indeed observed in heuristic uses and studies of the HMC algorithm \cite{Neal2011MCMCUsingHamiltonian}, recent efforts in proving theoretical bounds are mostly restricted to step sizes much shorter than the heuristic choices \cite{Chen2020FastMixingMetropolized,Chen2022OptimalConvergenceRate}.
In this work, we prove seemingly optimal bounds on the HMC algorithm (with leapfrog integrator) for the special case of Gaussian distributions.
This is the typical gateway to more complicated distributions such as logconcave or multimodal distributions.
Our implementation of HMC exploits long and randomized integration times.
This surpasses recent roadblocks on sampling Gaussian distributions using HMC with either short \cite{Chen2022OptimalConvergenceRate} or deterministic \cite{Lee2021LowerBoundsMetropolized} integration times.

Our bounds are stated most easily in the ``black box model'', where the goal is to sample from a density of the form $e^{-f(x)}$ for $x \in \R^d$, and we are given query access to both $f$ and its gradient $\nabla f$.
The Gaussian case further restricts $f$ to be a quadratic form $f(x) = \frac{1}{2} (x-\mu)^\top \Sigma^{-1} (x-\mu)$, where $\mu$ and $\Sigma$ are the (unknown) mean and covariance matrix of the Gaussian, respectively.
The \emph{condition number} of the Gaussian distribution is simply the condition number of $\Sigma^{-1}$. Throughout we assume that we are given bounds $0 < \alpha \leq \beta$ such that $\alpha I \preceq \Sigma^{-1} \preceq \beta I$ and we use $\kappa = \beta/\alpha$ as an upper bound on the condition number.
We prove the following theorem.
\begin{theorem*}[informal version of \cref{thm:adjusted-mixing-time}]
The Metropolis-adjusted HMC algorithm with leapfrog integrator can sample from a distribution $\eps$-close in total variation distance to a {$d$-dimensional} Gaussian distribution with condition number $\kappa$ using a total number of gradient evaluations\footnote{We use the $\tO(\cdot)$-notation to hide polylogarithmic factors in the problem parameters $d$, $\alpha$, $\beta$ and $\log(1/\eps)$.}
\[
\tO(\sqrt{\kappa} d^{1/4} \log(1/\eps)).
\]
\end{theorem*}

This theorem builds on an analysis of the \emph{unadjusted} HMC algorithm, for which we get a bound of $\tO(\sqrt{\kappa} d^{1/4}/\sqrt{\eps})$ on the total number of gradient evaluations.
Both bounds seem in line with expectation \cite{Duane1987HybridMonteCarlo,Neal2011MCMCUsingHamiltonian,Beskos2013OptimalTuningHybrid}, and we expect they are tight when using the usual leapfrog integrator for simulating the Hamiltonian dynamics.
Our algorithm surpasses the $\tOm(\kappa \sqrt{d})$ lower bound on the complexity of HMC for Gaussian sampling from \cite{Lee2021LowerBoundsMetropolized} by using \emph{randomized} integration times.
This avoids the well-known periodicity issues associated to a deterministic integration time.

Our work fits within the recent effort of proving non-asymptotic (and often tight) bounds on Markov chain algorithms for constrained distributions such as Gaussian distributions and, more generally, logconcave distributions (where $f$ is assumed to be convex).
Most of these efforts have focused on short step dynamics such as the ball walk, the Langevin algorithm, and HMC with short integration times.
The use of such ``local steps'' makes it easier to control the stability and acceptance probability of the algorithm.
However, the restriction to short step dynamics is also what slows down these algorithms, and this is what we avoid in our HMC algorithm.

Another motivation for studying Gaussian sampling is that the restriction to sampling Gaussian and logconcave distributions precisely parallels the restriction to quadratic and convex functions in optimization.
Nonetheless, a gap between the (first-order oracle) complexity for logconcave sampling and the $O(\min\{\sqrt{\kappa},d\})$ complexity for convex optimization is apparently deemed plausible.
More specifically, the authors in \cite{Lee2020LogsmoothGradientConcentration} suggest an $\Omega(\kappa)$ lower bound for logconcave sampling.
Our work shows that a sublinear $\kappa$-dependency is possible at least for the special case of Gaussian distributions, and we see it as evidence that a general $O(\sqrt{\kappa})$ bound for logconcave sampling is achievable.

Finally, as a direct application of our work, we mention the use of Gaussian sampling in the contextual multi-armed bandit problem \cite{Agrawal2012AnalysisThompsonSampling}.
A competitive exploration-exploitation strategy for this problem is called \emph{Thompson sampling}, which is an efficient manner of maintaining a posterior on the set of arms.
In the case of a linear payoff, as is considered in \cite{Agrawal2013ThompsonSamplingContextual}, the prior and posterior distributions are Gaussian distributions.
While recent works suggested the use of Langevin dynamics for Thompson sampling \cite{Mazumdar2020ApproximateThompsonSampling,Xu2022LangevinMonteCarlo}, our work suggests that the use of Hamiltonian Monte Carlo leads to faster algorithms.

\subsection{Background and prior work}

There is a vast body of work on the use of Markov chain algorithms for sampling from Gaussian and logconcave distributions.
These works mostly consider the (Metropolized) random walk or ball walk (MRW), the Metropolis-adjusted Langevin algorithm (MALA), and HMC.
We discuss those works most directly related to ours.

The earliest works focus on \emph{asymptotic} bounds or scaling limits on the performance as $d \to \infty$.
A $d^{1/4}$-scaling was already suggested in \cite{Duane1987HybridMonteCarlo,Kennedy1991AcceptancesAutocorrelationsHybrid,Beskos2013OptimalTuningHybrid} for the complexity of HMC with leapfrog integrator for Gaussians and logconcave product distributions.
This improves over the expected $d$- and $d^{1/3}$-scalings of MRW and MALA, respectively.
Indeed, in a recent work by Chewi et al.~\cite{Chewi2021OptimalDimensionDependence} it was proven that the complexity of MALA for standard Gaussian distributions (with $\kappa = 1$) scales as $\tO(d^{1/3})$.
For leapfrog HMC, the only non-asymptotic bounds scaling with $d^{1/4}$ seem to have been proven recently in \cite{Mangoubi2018DimensionallyTightBounds,Mou2021HighOrderLangevinDiffusion} for the \emph{unadjusted} HMC chain, and under additional regularity assumptions.
While these assumptions include Gaussians, the final complexities in these works scale at least with $\kappa^2$ and $1/\sqrt{\eps}$, and so scale much worse in terms of both $\kappa$ and $\eps$ compared to our bound.

An improved (linear) $\kappa$-dependency is obtained in recent works on MALA \cite{Dwivedi2018LogconcaveSamplingMetropolisHastings,Lee2020LogsmoothGradientConcentration,Wu2021MinimaxMixingTime} and HMC \cite{Chen2020FastMixingMetropolized}.
This seems optimal based on the $\tOm(\kappa \sqrt{d})$ lower bounds on MALA and HMC from \cite{Wu2021MinimaxMixingTime,Lee2021LowerBoundsMetropolized}, which even apply to the Gaussian case.
Such lower bounds typically follow from either restricting to \emph{short} integration times (as with MALA), which leads to diffusive behavior, or \emph{fixed} integration times, which can lead to periodic behavior in the HMC algorithm.
Either of these restrictions leads to an $\Omega(\kappa)$-dependency, and indeed we are not aware of any former non-asymptotic bounds on the mixing time achieving a \emph{sublinear} $\kappa$-dependency (while using a numerical integrator).
We sidestep these issues by using both \emph{long} and \emph{random} integration times.
Analyzing the resulting algorithm can be significantly more involved, and for this we restrict our analysis to the Gaussian case.
It however seems likely that this will form a gateway to proving $\sqrt{\kappa}$-scalings for general logconcave distributions.

The use of nonconstant integration times was also studied recently in the \emph{randomized} HMC algorithm by Bou-Rabee and Sanz-Serna \cite{Bou-Rabee2017RandomizedHamiltonianMonte}.
Similarly to our work, they motivate their algorithm by looking at the Gaussian case, and obtain similar scalings to our work for properties such as the autocorrelation time and mean displacement.
In follow-up works \cite{Deligiannidis2021RandomizedHamiltonianMonte,Lu2022ExplicitL2convergenceRate} (and \cite{
Wang2022AcceleratingHamiltonianMonte,Jiang2022DissipationIdealHamiltonian} restricted to the Gaussian case) bounds similar to ours are proven on the relaxation time.
However, all of these results are proven only for the \emph{idealized} case, and do not take into account the errors that arise from numerical integration.

Finally, for completeness we also mention that there are algorithms for Gaussian sampling that are \emph{not} based on Markov chains.
While these are generally incomparable (e.g., they require access to the precision or covariance matrix rather than gradient), we refer the interested reader to \cite{Vono2022HighDimensionalGaussianSampling}.

\subsection{Organization and proof overview}
In \cref{sec:intro} we formally introduce the problem and describe preliminaries related to Markov chains and Hamiltonian dynamics.
In particular, for the Gaussian case, we discuss how the numerical leapfrog integrator exactly integrates the Hamiltonian of a closely related Gaussian. In \cref{sec:idealized} we bound the mixing time of the HMC algorithm with an idealized integrator.
Using the observation about the leapfrog integrator, this mixing time extends to the ``unadjusted'' HMC algorithm, which is an exact HMC algorithm for a slightly perturbed Hamiltonian (and hence has a slightly perturbed stationary distribution). Finally, in \cref{sec:MAHMC}, we consider the Metropolis-adjusted HMC algorithm with leapfrog integrator.
This algorithm has the correct stationary distribution, but the mixing time might increase due to an additional accept-reject step.
We use high-dimensional concentration bounds (in particular, the Hanson-Wright inequality) to show that the acceptance rate is usually large.
This suffices to bound the mixing time through the use of $s$-conductance, which proves our main result.

\section{Problem definition and preliminaries} \label{sec:intro}

\subsection{Gaussian sampling} \label{sec:gaussian-sampling}

We consider a $d$-dimensional Gaussian distribution with unknown precision matrix $B$ (equal to the inverse of the covariance matrix, $B = \Sigma^{-1}$) and mean $\mu = 0$.\footnote{This is without loss of generality. Using $\tO(\sqrt{\kappa})$ gradient queries we can always determine the mean up to high precision and then translate the Gaussian to the origin.}
In such case, the Gaussian distribution is $\pi(x) \propto \exp(-f(x))$ with $f(x) = \frac{1}{2} x^\top B x$ for $x \in \R^d$ and $B$ a positive definite matrix. The algorithms we use (Hamiltonian Monte Carlo with a leapfrog integrator) are basis invariant, and so for ease of notation we will assume throughout that $B$ is diagonal with $B_{ii} = \omega_i^2$ for each $i \in [d]$. As input, we are given bounds $0 < \alpha \leq \beta$ such that $\alpha I \preceq B \preceq \beta I$, or, equivalently, $\alpha \leq \omega_i^2 \leq \beta$.
The condition number of $B$ is upper bounded by $\kappa = \beta/\alpha$ and we will also call this the condition number of $\pi$.
We assume \emph{first-order query access} to $f$, which means that a single query at a point $x \in \R^d$ provides both $f(x)$ and $\nabla f(x) = Bx$. The goal is to return a sample from a distribution that is $\eps$-close to $\pi$ in total variation distance, while making a minimal number of gradient queries to $f$.

\subsection{Markov chains on \texorpdfstring{$\R^d$}{Rd}}

Throughout we work with Markov chains whose behaviour can be described as follows: when at $x \in \R^d$ move to $y \in \R^d$ with probability density $T(x,y) \geq 0$. We identify the Markov chain with the transition kernel (density) $T: \R^d \times \R^d \to \R_+$. For a fixed $x \in \R^d$ we use $T_x$ to denote the probability distribution on $\R^d$ with density $T(x,\cdot)$.
Similarly (with some abuse of notation), we denote by $T_\mu$ the probability distribution on $\R^d$ with density $\int \mu(x) T(x,\cdot) \dif x$.
The $K$-step transition kernel $T^K$ is defined recursively via $T^{K}(x,y) = \int_{\R^d} T^{K-1}(x,z) T(z,y) \dif z$ for $K>1$. We say that $T$ satisfies the \emph{detailed balance condition} with respect to the probability density $\pi:\R^d \to \R_+$ if 
\[
\pi(x) T(x,y) = \pi(y) T(y,x) \quad \text{ for all } x,y \in \R^d.
\]
The associated Markov chain is called \emph{reversible}.

\subsection{Hamiltonian dynamics, harmonic oscillator and leapfrog integrator} \label{sec:harmonic oscillator intro}

At its core, Hamiltonian Monte Carlo makes moves by integrating Hamiltonian dynamics.
In general, these describe the evolution of a physical system parameterized by (generalized) \emph{positions} and (generalized) \emph{momenta}. For the purposes of this paper, we denote the former with $x \in \R^d$ and the latter with $v \in \R^d$. We sometimes refer to $v$ as the \emph{velocity}, which in classical physics is equal to the momentum of a unit mass. The Hamiltonian evolution of a $d$-dimensional system is governed by its Hamiltonian $\cH: \R^d \times \R^d \to \R$, which can be understood as the total energy of the system at position $x \in \R^d$ and with velocity $v\in \R^d$. The evolution of the system is described by the following equations:
\[
    \frac{\dif x}{\dif t} = \frac{\partial \cH(x, v)}{\partial v},\quad \frac{\dif v}{\dif t} = -\frac{\partial \cH(x, v)}{\partial x}.
\]
The simplest example is the (one-dimensional) harmonic oscillator with Hamiltonian $\cH(x, v) = \frac12 \omega^2 x^2 + \frac12 v^2$ for some given $\omega > 0$.
Its evolution is described by $\frac{\dif x}{\dif t} = v$ and $\frac{\dif v}{\dif t} = -\omega^2 x$, which can be solved analytically to yield
\begin{equation} \label{eq:harm-osc}
\begin{bmatrix} x(t) \\ v(t) \end{bmatrix}
= \begin{bmatrix}
\cos(\omega t) & \frac{1}{\omega} \sin(\omega t) \\
- \omega \sin(\omega t) & \cos(\omega t) \end{bmatrix}
\begin{bmatrix} x(0) \\ v(0) \end{bmatrix}
\end{equation}
A more interesting example is the $d$-dimensional harmonic oscillator. For a given positive (semi-)definite matrix $B \in \R^{d\times d}$, its Hamiltonian is $\cH(x, v) = \frac12 x ^\top B x + \frac12 v^\top v$, and its evolution is described by
\begin{equation}\label{eq:general Hamilton's eqs for HO}
    \frac{\dif x}{\dif t} = v,\quad \frac{\dif v}{\dif t} = -Bx.
\end{equation}
If $B$ has eigenvalues $\omega_i^2$ then in the eigenbasis of $B$ the system effectively decomposes into $d$ independent harmonic oscillators with frequencies $\omega_i$.
When analyzing our algorithms, it is often useful to assume that $B$ is diagonal, so we can treat each coordinate independently. Of course, the algorithms themselves remain basis-independent, and only require the aforementioned bounds $\alpha$ and $\beta$ on the eigenvalues $\omega_i^2$.

\subsubsection{Leapfrog integrator} \label{sec:leapfrog}
The leapfrog integrator, also known as the St\"ormer-Verlet method, is a well-known numerical integrator for Hamiltonian dynamics that uses two queries to $\frac{\partial \cH(x, v)}{\partial x}$ in each integration step. In the Gaussian case we have $\cH(x,v) = \frac12 x^\top Bx + \frac12 v^\top v$ and the propagator takes the following closed form: 
\begin{equation} \label{eq:matrix leapfrog}
    \begin{bmatrix}
    x^{(n+1)} \\
    v^{(n+1)}
    \end{bmatrix} = \begin{bmatrix}
        I - \frac{\delta^2}{2}B & \delta I \\
        -\delta B (I - \frac{\delta}{4}B) & I - \frac{\delta^2}{2} B
    \end{bmatrix} \begin{bmatrix}
    x^{(n)} \\
    v^{(n)}
    \end{bmatrix},
\end{equation}
where $\delta>0$ is a parameter used to describe the integration time. See for example~\cite[Sec.~2.6]{Leimkuhler2005SimulatingHamiltonianDynamics} for details. We will exploit that, similarly as for the idealized Hamiltonian dynamics, the leapfrog dynamics also decouple in the diagonal basis of $B$. Hence, as before, we can assume without loss of generality that $B$ is diagonal with entries $0 < \alpha \leq \omega_i^2 \leq \beta$, and the leapfrog integrator can be interpreted as integrating $d$ independent harmonic oscillators. To understand the leapfrog integrator we can thus restrict to a single harmonic oscillator with parameter $\omega$. 

The propagator from \cref{eq:matrix leapfrog} has eigenvalues
\[
\lambda^\pm
= 1 - \frac{\delta^2 \omega^2}{2}
    \pm \im \delta \omega \sqrt{1 - \frac{\delta^2 \omega^2}{4}}. 
\]
If $\delta^2 \omega^2 \leq 4$, we can set $\lambda^\pm = e^{\pm \im \varphi}$,
where $\varphi \in [0,\pi]$ is uniquely defined by $\cos(\varphi) = 1 - \frac{\delta^2 \omega^2}{2}$ and $\sin(\varphi) = \delta \omega \sqrt{1 - \frac{\delta^2 \omega^2}{4}}$.
We can use $\varphi$ to rewrite the propagator as a rotation with angle $\varphi$
\[
\begin{bmatrix} \cos(\varphi) & \frac{1}{\hat\omega} \sin(\varphi) \\
- \hat\omega \sin(\varphi) & \cos(\varphi) \end{bmatrix}, \qquad \text{ where } \hat\omega
= \omega \sqrt{1 - \frac{\delta^2\omega^2}{4}}.
\]
Comparing this with \eqref{eq:harm-osc}, we see that the leapfrog trajectory \emph{exactly} follows the Hamiltonian dynamics for the modified Hamiltonian $\hat \cH$ given by 
\begin{equation*} 
\hat \cH(x,v)
= \frac{1}{2} \hat\omega^2 x^2 + \frac{1}{2} v^2. 
\end{equation*}
Indeed, if $(\hat x(t), \hat v(t))$ is the solution of Hamilton's equations with Hamiltonian $\hat \cH(x,v)$ and initial conditions $(\hat x(0) = x_0,\hat v(0) = v_0)$, then the $n$th point on the leapfrog trajectory equals
\begin{align*} 
\begin{bmatrix} \hat x^{(n)} \\ \hat v^{(n)} \end{bmatrix} = \begin{bmatrix} \cos(n \varphi) & \frac{1}{\hat\omega} \sin(n \varphi)  \\
- \hat\omega \sin(n \varphi) & \cos(n \varphi) \end{bmatrix}
 \begin{bmatrix} \hat x_0 \\ \hat v_0 \end{bmatrix} 
&= \begin{bmatrix} \cos(\hat\omega t_n) & \frac{1}{\hat\omega} \sin(\hat\omega t_n) \\
- \hat\omega \sin(\hat\omega t_n) & \cos(\hat\omega t_n) \end{bmatrix}
 \begin{bmatrix} \hat x_0 \\ \hat v_0 \end{bmatrix}
= \begin{bmatrix} \hat x(t_n) \\ \hat v(t_n) \end{bmatrix},
\end{align*}
where $t_n = n \varphi/\hat\omega$. We can now easily check that the difference between $\cH$ and $\hat \cH$ is 
\begin{equation*} 
\cH(x,v) - \hat\cH(x,v) = \frac{\delta^2 \omega^4 x^2}{8}.
\end{equation*}

By our former remark, this observation extends to general $d$-dimensional harmonic oscillators and the corresponding leapfrog integrator \eqref{eq:matrix leapfrog}: we define $\hat B$ by replacing $\omega_i$ by $\hat \omega_i$ for each eigenvalue of $B$, where  
\begin{equation} \label{eq:hat omega}
\hat \omega_i := \omega_i \sqrt{1-\frac{\delta^2 \omega_i^2}{4}},
\end{equation}
and we set $\hat \cH(x,v) = \frac{1}{2} x^\top \hat B x + \frac{1}{2} v^\top v$.
The leapfrog integrator is then an exact integrator for $\hat \cH$ and we have that
\begin{equation} \label{eq: distance H and hat H}
    \cH(x,v) - \hat \cH(x,v) = \frac{\delta^2}{8} \sum_{i \in [d]} \omega_i^4 x_i^2.
\end{equation}

Finally we introduce the following notation: the tuple $(x',v')= \LF(x,v,t,\delta)$ is defined as the (position, momentum)-vector after taking $t/\delta$ leapfrog integration steps for Hamiltonian $\cH$ with stepsize~${0 \leq \delta \leq 1/\sqrt{\beta}}$.\footnote{We will always apply this with $t/\delta \in \N$.}

\section{Idealized and unadjusted HMC} \label{sec:idealized}

We first analyze an idealized version of HMC, \cref{alg:ideal-HMC}, where we assume that we can exactly integrate the Hamiltonian dynamics. We use long and random integration times. In order to later apply the results from this section in the setting of a numerical integrator, we will use uniformly random integration times $t \sim U(\cT)$ from a \emph{finite} set $\cT$.
We will require only that for all $0<\alpha \leq \omega^2 \leq \beta$, we have that $\P_{t \sim U(\cT)}\big[|\cos(\omega t)| \leq 0.9\big] \geq 1/2$.
In the following lemma we show that this is satisfied for a simple choice of $\cT$. 

\begin{lemma} \label{lem: discrete T proof}
Let $0<\sqrt{\alpha} \leq \sqrt{\beta}$.
If $0<\delta \leq \pi/(20 \sqrt{\beta})$ and 
\begin{equation} \label{eq: def cT}
\cT = \{k \cdot \delta \mid k \in \N, \  k\cdot \delta<10\pi/\sqrt{\alpha}\}
\end{equation}
then we have for all $\omega \in [\sqrt{\alpha},\sqrt{\beta}]$ that
\begin{equation} \label{lem: discrete T}
\P_{t \sim U(\cT)}\big[|\cos(\omega t)| \leq 0.9\big] \geq 1/2.
\end{equation}
\end{lemma}
\begin{proof}
First, we prove that if $\zeta>\eta\geq 0$, $\omega > 0$, and $\widetilde \cT = \{\eta + n \zeta: n \in \N, \eta + n \zeta \leq \frac{\pi}{2\omega}\}$ with $|\widetilde \cT|\geq 10$, then for $t$ chosen uniformly from $\widetilde \cT$ we have 
    \begin{equation} \label{eq: prob for tilde T}
        \P_{t \sim U(\widetilde \cT)} \left\{ |\cos(\omega t)| \leq 0.9 \right\} \geq 3/5.
    \end{equation}
To see this, note that $\zeta \leq \left\lfloor \frac{\pi}{2 \omega (|\widetilde \cT| - 1)} \right\rfloor$ implies that
    \begin{align*}
        \P_{t \sim U(\widetilde \cT)} \left\{ |\cos(\omega t)| \leq 0.9 \right\} &= \P_{t \sim U(\widetilde\cT)} \left\{ t \geq \frac{1}{\omega} \arccos(0.9) \right\} \\
        &\geq \frac{1}{|\widetilde\cT|}  \left\lfloor \frac{\pi/2 - \arccos(0.9)}{\omega \zeta} \right\rfloor \\
        &\geq \frac{1}{|\widetilde\cT|} \left\lfloor \left(1 - 2\arccos(0.9)/\pi\right) (|\widetilde\cT| - 1) \right\rfloor.
    \end{align*}
    The last quantity is at least $3/5$ for $|\widetilde\cT| \geq 10$.

We now make use of the above to show the desired bound for the set $\cT$ defined in \cref{eq: def cT}. Let $\omega$ be such that $\sqrt{\alpha} \leq \omega \leq \sqrt{\beta}$. Note that $|\cos(\omega t)|$ is periodic with period $\frac{\pi}{2\omega}$. We write $\cT$ as the disjoint union 
 \[
 \cT = \bigcup_{n=1}^N \left(\cT \cap \big[\frac{(n-1)\pi}{2\omega},\frac{n\pi}{2\omega}\big]\right)
 \]
 where $N$ is the least integer such that $\frac{N\pi}{2\omega}>\frac{10\pi}{\sqrt{\alpha}}$, i.e., $N = \left\lfloor\frac{20\omega}{\sqrt{\alpha}}\right\rfloor$. Note that $N \geq 20$. Since $\delta \leq \frac{\pi}{20 \sqrt{\beta}}$ and $\omega \leq \sqrt{\beta}$, the first $N-1$ such intervals contain at least 
 \[
 \left \lfloor \frac{\pi}{2 \omega \delta}\right \rfloor
 \geq 10
 \]
 equally spaced points.
Now note that the subset $\cT \cap [\frac{(n-1)\pi}{2\omega},\frac{n\pi}{2\omega}])$ takes precisely the form as considered at the start of the proof, and we just proved that $|\cT \cap [\frac{(n-1)\pi}{2\omega},\frac{n\pi}{2\omega}])| \geq 10$.
Hence, \cref{eq: prob for tilde T} shows that for each of these $N-1$ intervals we have 
 \[
 \P_{t \sim U(\cT \cap [\frac{(n-1)\pi}{2\omega},\frac{n\pi}{2\omega}])} \big[|\cos(\omega t)| \leq 0.9\big] \geq \frac{3}{5}.
 \]
 Given that there are $N \geq 20$ intervals in total, we get
 \[
\P_{t \sim U(\cT)}\big[|\cos(\omega t)| \leq 0.9\big]
\geq \frac{N-1}{N} \frac{3}{5}
\geq \frac{19}{20} \frac{3}{5} \geq \frac{1}{2}. \qedhere
\]
\end{proof}

We now formulate the HMC algorithm using this definition of $\cT$.

\begin{algorithm}[hbt!]
\caption{\textbf{Markov kernel $P$} (idealized HMC with random integration time)}  \label{alg:ideal-HMC}
\Input{$x\in \R^d$, $\cT$ as in \cref{eq: def cT}} 
\Output{$x' \in \R^d$}
Draw $v \sim \cN(0,I_d)$ and $t \sim U(\cT)$\;
Define $x'$ by following Hamiltonian dynamics for $\cH$ for time $t$, starting from $(x,v)$\; 
\end{algorithm}

It is well known that idealized HMC with a fixed integration time has the desired stationary distribution $\pi$ whose density at $(x,v)$ is related to the Hamiltonian $\cH(x,v) = \frac{1}{2} x^\top \diag(\bm \omega) x + \frac{1}{2} v^\top v$, i.e., $\pi(x,v) \propto {\exp(-\cH(x,v))}$ (cf.~\cite{Duane1987HybridMonteCarlo,Neal1996BayesianLearningNeural,Vishnoi2021IntroductionHamiltonianMonte}). From this it follows that also $P$ has stationary distribution~$\pi$. In \cref{sec:ideal-HMC-MT} we show that $P$ has a small mixing time. We then extend this result to the setting where we use a numerical integrator (leapfrog) instead of the idealized time evolution according to Hamiltonian dynamics. For this we use the fact (cf.~\cref{sec:leapfrog}) that the leapfrog integrator applied to $\cH(x,v)$ can be viewed as an exact integrator for the Hamiltonian dynamics of a modified Hamiltonian $\hat \cH(x,v)$. By bounding the distance between $\pi$ and $\hat \pi \propto \exp(-\hat \cH(x,v))$, we output a distribution that is $\eps$-close to $\pi$ in total variation distance using a number of gradient evaluations that scales as $\tO(\sqrt{\kappa} d^{1/4}/\sqrt{\eps})$, see \cref{sec:unadjusted}.

\subsection{Idealized HMC} \label{sec:ideal-HMC-MT}

Let $P^t_x$ denote the density function of the proposal distribution from $x \in \R^d$, conditioned on having picked $t \in [0,T]$. Using the explicit expression \cref{eq:harm-osc}, we can expand it as
\begin{align} \notag
    P_x^t(z) &= \mathbb{P}_{v \sim \cN(0,1)}\left[\cos(\omega_i t) x_i + \frac{1}{\omega_i} \sin(\omega_i t) v_i = z_i \quad \forall i \in [d]\right] \\
    &= (2\pi)^{-d/2}\prod_{i \in [d]} \frac{\omega_i}{|\sin(\omega_i t)|} \exp\left(-\frac{1}{2}\left(\frac{z_i-\cos(\omega_i t) x_i}{\frac{1}{\omega_i} \sin(\omega_i t)}\right)^2\right). \label{def: Pxt}
\end{align}
The probability density with which idealized HMC moves from $x$ to $z$ is then given by $P_x(z) = \frac{1}{|\cT|} \sum_{t \in \cT} P_x^t(z)$.

We analyze the convergence in total variation distance by explicitly writing out the distribution $P^K$ obtained by taking $K$ steps of the idealized HMC method. If we condition on the choice of random integration times in step 2 of \cref{alg:ideal-HMC}, then the resulting distribution is again a normal distribution. Indeed, let $(v^{(1)},\dots,v^{(K)})$, $(t_1,\dots,t_K)$ and $(x^{(1)},\dots,x^{(K)})$ denote the velocities, integration times and positions, respectively, encountered during the first $K$ steps.
By repeatedly applying \eqref{def: Pxt}, we can express
\begin{align*}
x^{(K)}_i
&= x^{(K-1)}_i \cos(\omega_i t_K) + \frac{1}{\omega_i} \sin(\omega_i t_K) v^{(K)} \\
&= x^{(0)}_i \left(\prod_{k=1}^K \cos(\omega_i t_k)\right) +\frac{1}{\omega_i}\sum_{k=1}^K v^{(k)} \sin(\omega_i t_k) \left(\prod_{j=k+1}^K \cos(\omega_i t_j)\right).
\end{align*}
For a fixed tuple $\bm t = (t_1,\dots,t_K) \in \cT^K$ of integration times, but \emph{random} choices $(v^{(1)},\dots,v^{(K)}) \sim \cN(0,I_d)^K$ of momenta, we can argue that this describes a Gaussian distribution, which we denote by $P_x^{\bm t}$. First, note that~$P_x^{\bm t}$ is a product distribution: $P_x^{\bm t}(z) = \prod_{i \in [d]} P_x^{\bm t,i}(z_i)$ where we use~$P_x^{\bm t,i}$ for the marginal distribution of~$P_x^{\bm t}$ with respect to the $i$-th coordinate. 
Then, note that $P_x^{\bm t,i}$ describes a sum of Gaussians with the same mean, and hence forms again a Gaussian.
We formalize this in the next lemma.

\begin{lemma}
Let $\bm t \in \cT^K$, $\omega>0$, $x \in \R$, and consider 
\[
z = x \left(\prod_{k=1}^K \cos(\omega t_k)\right) +\frac{1}{\omega}\sum_{k=1}^K v^{(k)} \sin(\omega t_k) \left(\prod_{j=k+1}^K \cos(\omega t_j)\right)
\]
where $v^{(k)} \sim \cN(0,1)$ for each $k \in [K]$. Then $z \sim \cN(x \prod_{k=1}^K \cos(\omega t_k), \frac{1}{\omega^2}(1-\prod_{k=1}^K \cos(\omega t_j)^2))$. 
\end{lemma}
\begin{proof}
It is clear that $\Exp[z] = x \prod_{k=1}^K \cos(\omega t_k)$. The sum of Gaussian random variables is again distributed according to a Gaussian whose variance is the sum of the individual variances. That is, 
\begin{align*}
\Exp[(z-\Exp[z])^2] &= \frac{1}{\omega^2}\sum_{k=1}^K \sin(\omega t_k)^2 \left(\prod_{j=k+1}^K \cos(\omega t_j)^2\right)\\
&= \frac{1}{\omega^2}\sum_{k=1}^K (1-\cos(\omega t_k)^2) \left(\prod_{j=k+1}^K \cos(\omega t_j)^2\right) \\
&= \frac{1- \prod_{j=1}^K \cos(\omega t_j)^2}{\omega^2}. \qedhere
\end{align*}
\end{proof}

If the term $\prod_{k=1}^K \cos(\omega t_k)$ is sufficiently small, then $P^{\bm t}_x$ is close to $\pi$. \cref{lem: discrete T proof} and Hoeffding's inequality show that for a random tuple $\bm t = (t_1,\dots,t_K) \sim U(\cT^K)$ this term will indeed be small.
Then we use this to prove convergence of the proposal distribution to $\pi$.

\begin{lemma} \label{lem: exponential decrease}
Let $0<\alpha \leq \omega^2 \leq \beta$ and $\cT$ as in \cref{lem: discrete T}. Then
\[
\P_{\bm t \sim U(\cT^K)}\left[\Big|\prod_{k=1}^K \cos(\omega t_k)\Big| \geq 0.9^{K/4}\right] \leq \exp(-K/8).
\]
\end{lemma}
\begin{proof}
Let $\bm t = (t_1,\dots,t_K)$ with $t_k \sim U(\cal T)$, and define the i.i.d.~Boolean variables $X_k$ as indicating whether $|\cos(\omega t_k)| \leq 0.9$.
Define $\rho = \P[X_k = 1]$.
By \cref{lem: discrete T proof} we know that $\rho \geq 1/2$.
By the multiplicative Chernoff bound this implies that
\[
\P\left[ \sum_{k=1}^K X_k \leq \frac{K}{4} \right]
\leq \P\left[ \sum_{k=1}^K X_k \leq \frac{K \rho}{2} \right]
\leq \exp\left(-\frac{K}{8}\right).
\]
It remains to note that if $\sum_{k=1}^K X_k > K/4$ then $\Big|\prod_{k=1}^K \cos(\omega t_k)\Big| < 0.9^{K/4}$, and this implies that
\[
\P_{\bm t \sim U(\cT^K)}\left[\Big|\prod_{k=1}^K \cos(\omega t_k)\Big| \geq 0.9^{K/4}\right]
\leq \P\left[ \sum_{k=1}^K X_k \leq \frac{K}{4} \right]. \qedhere
\]
\end{proof}

Using the above lemma, we show that the proposal distributions $P_x^K(z) := P^K(x,z)$ and $P_y^K := P^K(y,z)$ are close provided that $x$ and $y$ are close.  
\begin{proposition} \label{prop: tv bound on K step HMC}
For every $x,y \in \R^d$, if 
\[
K \geq 38 \log\left(\frac{d(2+\sqrt{\beta}\|x-y\|_\infty)}{\eps} \right),
\]
then, with $P$ the kernel of idealized HMC, we have 
\[
\|P_x^K -  P_y^K\|_{\TV} \leq \eps.
\]
\end{proposition}
\begin{proof}
Recall that $P_x^K = \frac{1}{|\cT|^K} \sum_{\bm t \in \cT^K} P_x^{\bm t}$ and $P_x^{\bm t} = \prod_{i \in [d]} P_x^{\bm t,i}$ is a product distribution.
Hence, we can twice apply a triangle inequality to obtain 
\begin{align} \notag
    \|P_x^K -  P_y^K\|_{\TV} &\leq \frac{1}{|\cT|^K} \sum_{\bm t \in \cT^K} \|P_x^{\bm t} - P_y^{\bm t}\|_{\TV}  \\
    &\leq \sum_{i \in [d]} \frac{1}{|\cT|^K} \sum_{\bm t \in \cT^K} \|P_x^{\bm t,i} - P_y^{\bm t,i}\|_{\TV} \label{eq:triangle ineq on TV}
\end{align}
Now let $\delta= \frac{1}{\sqrt{2}}\min\big\{1,\frac{\eps}{2d(2+\sqrt{\beta}\|x-y\|_{\infty})}\big\}$ and $K \geq 38 \log(1/\delta)$.
We will invoke \cref{lem: exponential decrease}.
By our choice of $K$ we have that $0.9^{K/4} \leq \delta$ and $\exp(-K/8) \leq \eps/(2d)$, and so the lemma ensures that
\[
\P_{\bm t \sim U(\cT^K)}\left[\Big|\prod_{k=1}^K \cos(\omega_i t_k)\Big| \geq \delta \right] \leq \frac{\eps}{2d}
\]
for each $i \in [d]$. Hence for each coordinate $i \in [d]$ we have 
\begin{align} \notag
   \frac{1}{|\cT|^K} \sum_{\bm t \in \cT^K} \|P_x^{\bm t,i} - P_y^{\bm t,i}\|_{\TV} &\leq \frac{\eps}{2d} + \frac{1}{|\cT|^K} \sum_{\bm t \in \cT^K: |\prod_{k=1}^K \cos(\omega_i t_k)| \leq \delta} \|P_x^{\bm t,i} - P_y^{\bm t,i}\|_{\TV} \\
  &\leq \frac{\eps}{2d} + (1-\frac{\eps}{2d})
  |x_i - y_i| \delta \sqrt{2}\omega_i
  \label{eq:per coordinate TV bound}
  \leq \eps,
\end{align}
where we use that for $\bm t \in \cT^K$ for which $\big|\prod_{k=1}^K \cos(\omega_i t_j)\big| \leq \delta \leq\frac{1}{\sqrt{2}}$, the proposal distributions $P_x^{\bm t,i}$ and $P_y^{\bm t,i}$ are univariate Gaussians with means $\mu_x, \mu_y$ that satisfy $|\mu_x - \mu_y| \leq \delta|x_i - y_i|$, and both have variance $\sigma^2 \geq \frac{1-\delta^2}{\omega_i^2} \geq \frac{1}{2\omega_i^2}$.  (For univariate Gaussians one has $\|\cN(\mu_,\sigma^2) - \cN(\mu_2,\sigma^2)\|_{\TV}<|\mu_1-\mu_2|/\sigma$.)  
Combining \cref{eq:triangle ineq on TV,eq:per coordinate TV bound} we obtain $\|P_x^K-P_y^K\|_{\TV} \leq \eps$. 
\end{proof}

This bound then easily leads to a bound on the total variation distance between $P_x^K$ and $\pi$ for $x$ that is sufficiently close to $0$, and this is the main conclusion of this section.
\begin{theorem}[Idealized HMC] \label{cor:idealized-mixing}
There exists a constant $C>0$ such that for every $x \in \R^d$, if 
\[
K \geq C\log\left(\frac{d\kappa(\sqrt{\alpha}\norm{x}_\infty + 1)}{\eps}\right),
\]
then, with $\pi \propto \exp(-\frac{1}{2} x^\top Bx)$ and $P$ the kernel of idealized HMC, we have 
\[
\|P_x^K - \pi\|_{\TV} \leq \eps.
\]
\end{theorem}
\begin{proof}
We write $\pi = \int_{\R^d} \delta_y \dif\pi(y)$.
Using that $\pi$ is stationary for $P$ (and hence $P^K$), we also have that $\pi = \int_{\R^d} P_y^K \dif\pi(y)$.
Now we apply Jensen's inequality:
\begin{align*}
\|P_x^K - \pi\|_{\TV} &\leq \int_{y \in \R^d} \|P_x^K - P_y^K\|_{\TV} \dif\pi(x) \\
&\leq \pi(\{y:\|y\|>\eta \}) + \int_{y \in \R^d: \|y\|\leq \eta} \|P_x^K - P_y^K\|_{\TV} \dif\pi(x). 
\end{align*}
We use \cref{lem:concentration Egamma} to choose an $\eta$ that is sufficiently large to ensure that $\pi(\{y:\|y\|>\eta\}) \leq \eps/2$. In particular, using the notation of that lemma, for $\gamma = \Theta(\log(1/\eps))$ we know that $\pi(E_\gamma) \geq 1 - \eps / 2$, and we can bound the norm of each $y \in E_\gamma$ as
\[
    \alpha^2 \norm{y}^2 \leq y^\top \diag(\vec{\omega})^4 y \leq \sum_i \omega_i^2 + \gamma \sqrt{\sum_i \omega_i^4} \leq d \beta + \gamma \beta \sqrt{d},
\]
which yields the bound $\norm{y} \leq \sqrt{\frac{(\gamma+1) \kappa d}{\alpha}}$ for $y \in E_\gamma$. We use this to bound the quantity $\frac{d \sqrt{\beta}\norm{x-y}_\infty}{\eps}$ as follows
\begin{align*}
    \frac{d \sqrt{\beta}\norm{x-y}_\infty}{\eps} &\leq \frac{d \sqrt{\beta}(\norm{x}_\infty + \sqrt{(\gamma+1) \kappa d / \alpha})}{\eps} \\
    &\leq \frac{d \sqrt{\kappa}(\sqrt{\alpha}\norm{x}_\infty + \sqrt{(\gamma+1) \kappa d})}{\eps} \\
    &\leq \frac{d^{3/2} \kappa \sqrt{(\gamma+1)} (\sqrt{\alpha} \norm{x}_\infty + 1)}{\eps}.
\end{align*}
This shows there exists a $C>0$ such that for $K \geq C\log(\frac{d\kappa(\sqrt{\alpha}\norm{x}_\infty + 1)}{\eps})$ we have $K \geq 38\log\left(\frac{d(2+\sqrt{\beta}\|x-y\|_\infty)}{\eps/2} \right)$ and therefore \cref{prop: tv bound on K step HMC} implies that $\|P_x^K - P_y^K\|_{\TV} \leq \eps/2$ for all $x,y \in \R^d$ with $\|x-y\|_\infty \leq \eta + \|x\|_\infty$. Combining these two bounds shows that $\|P_x^K - \pi\|_{\TV} \leq \eps$.
\end{proof}

\subsection{Unadjusted HMC} \label{sec:unadjusted}

The results from the previous section extend from the idealized setting where one can integrate exactly, to the setting where one uses the leapfrog integrator. 

\begin{algorithm}[H]
\caption{\textbf{Markov kernel $\hat Q$} (leapfrog HMC with random integration time) }
\label{alg:lf-HMC}
\Input{$x \in \R^d$, stepsize $\delta \leq 1/\sqrt{\beta}$, $\cT$ as in \cref{eq: def cT}} 
\Output{$x' \in \R^d$}
Draw $v \sim \cN(0,I_d)$ and move from $x$ to $(x,v)$ \;
Draw $t \sim U(\cT)$ and set $(x',v') = \LF(x,v,t,\delta)$ \;
\end{algorithm}

As discussed in \cref{sec:leapfrog}, the leapfrog dynamics correspond to Hamiltonian dynamics for a slightly modified Hamiltonian $\hat \cH$. Bounding the distance between the stationary distribution $\hat \pi$ and $\pi$ leads to the following $\poly(1/\eps)$-algorithm for sampling from a distribution $\eps$-close to $\pi$. 
\begin{proposition}[Unadjusted HMC] \label{lem:mixing-unadjusted}
There exist constants $C, C' >0$ such that for every $x \in \R^d$, if 
\[
K \geq C\log\left(\frac{d\kappa(\sqrt{\alpha}\norm{x}_\infty + 1)}{\eps}\right)
\quad \text{ and } \quad
\delta \leq C' \frac{\sqrt{\eps}}{\sqrt{\beta} d^{1/4}},
\]
then
\[
\| \hat Q_x^K - \pi\|_{\TV} \leq \eps
\]
where $\pi(x) \propto \exp(-\frac{1}{2} x^\top B x)$ and $\hat Q$ is the kernel of the unadjusted leapfrog HMC chain with step size $\delta$. A sample from $\hat Q_x^K$ can be obtained using $O( \frac{\sqrt{\kappa} d^{1/4}K}{\sqrt{\eps}} )$ gradient evaluations.
\end{proposition}
\begin{proof}
By our discussion of the leapfrog integrator in \cref{sec:leapfrog}, we know that $\hat Q$ corresponds to the idealized HMC algorithm for the modified Hamiltonian $\hat\cH$. Here we assume $\delta^2 \omega_i^2 \leq 4$ for all $i \in [d]$, i.e., $\delta \leq \frac{1}{\sqrt{\beta}}$. 
It thus follows from \cref{cor:idealized-mixing} that if we start from $x \in \R^d$ and take $K \geq C\log\left(\frac{d\kappa(\sqrt{\alpha}\norm{x}_\infty + 1)}{\eps}\right)$ steps of the chain $\hat Q$, for an appropriate constant $C>0$, then it returns a distribution that is $\eps/2$-close to the modified stationary $\hat\pi$ defined as 
\[
\hat \pi(x) \propto \exp(-\frac{1}{2} x^\top \hat B x).
\] 
Using that $\hat \pi$ and $\pi$ are both multivariate Gaussians, one can show (see \cref{lem:dist-mod} below for completeness)
\begin{equation*} 
\norm{\pi - \hat\pi}_{\TV} \leq \frac38 \delta^2\sqrt{\sum_i \omega_i^4}  \leq \frac38 \delta^2 \beta \sqrt{d}.
\end{equation*}
Hence by choosing a sufficiently small stepsize $\delta \in O(\sqrt{\eps}/(\sqrt{\beta} d^{1/4}))$, we have that $\| \hat\pi - \pi \|_{\TV} \leq \eps/2$. Together this shows that the resulting distribution after $K$ steps will be $\eps$-close to $\pi$.

It remains to bound the complexity of the algorithm. A single leapfrog step requires 2 gradient evaluations, and so a single step of the Markov chain $\hat Q$ requires $t/\delta \in O(\sqrt{\kappa} d^{1/4}/\sqrt{\eps})$ gradient evaluations.
Applying $K$ steps of the Markov chain yields a total number of gradient evaluations
\[
O\left( \frac{\sqrt{\kappa} d^{1/4} K}{\sqrt{\eps}} \right). \qedhere 
\]
\end{proof}

\begin{lemma} \label{lem:dist-mod}
Let $\pi(x) \propto \exp(-x^\top \diag(\bm \omega) x/2)$, $\hat \omega_i = \omega_i \sqrt{1-\frac{\delta^2 \omega_i^2}4}$ and $\hat \pi(x) \propto \exp(-x^\top \diag(\bm{\hat \omega}) x/2)$.
Then 
\[
\norm{\pi - \hat\pi}_{\TV} \leq \frac38 \delta^2\sqrt{\sum_i \omega_i^4}  \leq \frac38 \delta^2 \beta \sqrt{d}.
\]
\end{lemma}
\begin{proof}
For multivariate mean-zero Gaussians we have the following bound \cite{Devroye2022TotalVariationDistance}:
\begin{equation} \label{eq:gaussian mean zero tv}
\norm{\cN(0, \Sigma_1) - \cN(0, \Sigma_2)}_{\TV} \leq \frac32 \min \left\{1, \norm{\Sigma_1^{-1}\Sigma_2 - I}_F\right\}.
\end{equation}
Applying this bound for $\Sigma_1 = \diag(\bm{\hat \omega})$ and $\Sigma_2 = \diag(\bm \omega)$ we get
\[
    \norm{\pi - \tilde\pi}_{\TV} \leq \frac32 \sqrt{ \sum_i \left( \left(1- \frac{\delta^2 \omega_i^2}{4}\right) -1 \right)^2 }= \frac38 \delta^2\sqrt{\sum_i \omega_i^4} \leq \frac38 \delta^2 \beta \sqrt{d}. \qedhere
\]
\end{proof}

\section{Metropolis-Adjusted HMC} \label{sec:MAHMC}

Here we study the Metropolis-adjusted HMC algorithm.
The algorithm applies a Metropolis filter to correct for the numerical errors of the integrator.
This ensures that the algorithm has the correct stationary distribution, and leads to an overall improved error dependence.

\begin{algorithm}[H]
\caption{\textbf{Markov kernel $Q$} (Adjusted leapfrog HMC with random integration time)} \label{alg:adjusted-HMC}
\Input{$x \in \R^d$, stepsize $\delta \in O(1/(\sqrt{\beta} d^{1/4}))$, $\cT := \{k\cdot \delta \mid k\in \N, \ k \cdot \delta < 10\pi/\sqrt{\alpha}\}$}
\Output{$x' \in \R^d$}
Draw $v \sim \cN(0,I_d)$ and move from $x$ to $(x,v)$ \label{step1} \;
Draw $t \sim U(\cT)$  and set $(x', v') = \LF(x,v,t,\delta)$ \;
Accept with probability \label{step4}
    \[
    \min\Big\{ 1, \exp\big( -\cH(x',-v')+\cH(x, v) \big) \Big\}
    \]
and return $x'$.
Otherwise return $x' = x$\;
\end{algorithm}

We make a few observations about the adjusted HMC algorithm.
\begin{lemma} \label{lem:reversible}
The Markov kernel $Q$ defined in \cref{alg:adjusted-HMC} has the following properties:
\begin{enumerate} 
\item Kernel $Q$ is reversible with respect to the stationary distribution $\pi(x) \propto \exp(-\frac{1}{2} x^\top B x)$.
\item
The acceptance probability in \cref{step4} is a function of only $x$ and $x'$:
\[
\min\Big\{ 1, \exp\big( -\cH(x',-v')+\cH(x, v) \big) \Big\}
= \min\bigg\{ 1, \exp\bigg(\frac{\delta^2}{8} \sum_{i \in [d]} \omega_i^4 (x_i^2 - {x_i'}^2)\bigg) \bigg\}
\eqqcolon A(x,x'),
\]
and we can rewrite $Q_x(x') = \hat Q_x(x') A(x,x')$ for $x \neq x'$.
\end{enumerate}
\end{lemma}
\begin{proof}[Proof of \cref{lem:reversible}, part 1]
This fact is well known for fixed integration times.
Here we prove that it also holds for \emph{randomized} integration times.

We prove first that $Q$ leaves the distribution $\pi(x) \propto \exp(-x^\top B x/2)$ invariant.
To this end, we look at the larger \emph{phase space}.
Starting from $x \sim \pi$, the state $(x,v)$ in step 2 is distributed according to the distribution
\[
\tilde\pi(x,v)
\propto \exp(-x^\top B x/2 - v^\top v/2)
= \exp(-\cH(x,v)).
\]
It remains to prove that steps 2.~and 3.~leave $\tilde\pi$ invariant.
Let $T$ denote the kernel of the proposal generated in step 2.~(i.e., proposal $(x',-v')$ has density $T((x,v),(x',v'))$).
First we note that $T$ is \emph{symmetric}, i.e., $T((x,v),(x',v')) = T((x',v'),(x,v))$.
To see this, recall that leapfrog integration is reversible in the sense that $\LF(x,v,t/\delta,\delta) = (x', v')$ implies that $\LF(x', -v', t/\delta, \delta) = (x, -v)$, and hence
\begin{align*}
    T((x, v), (x', v'))
    &= \frac{1}{|U(\cT)|} \sum_{t \in U(\cT)} \1\left\{ \LF(x, v, t) = (x', -v') \right\} \\
    &= \frac{1}{|U(\cT)|} \sum_{t \in U(\cT)} \1\left\{ \LF(x', v', t) = (x, -v) \right\}
    = T((x',v'),(x,v)).
\end{align*}
Then, note that step 3.~effectively implements a Metropolis filter w.r.t.~distribution $\tilde\pi$, which has acceptance probability
\[
A((x,v),(x',v'))
= \min\left\{1, \frac{\tilde\pi(x',v')}{\tilde\pi(x,v)} \right\}
= \min\Big\{ 1, \exp\big( -\cH(x',-v')+\cH(x, v) \big) \Big\}.
\]
It is then a direct consequence that steps 2.~and 3.~leave $\tilde\pi$ invariant as well.

Next, we show that $Q$ is in fact \emph{reversible} with respect to $\pi$, i.e.,
\[
\pi(x) Q(x,x')
= \pi(x') Q(x',x),\quad \text{ for all } x, x' \in \R^d.
\]
To do this, we use the fact that for all $x,v \in \R^d$, the density $\tilde\pi(x, v)$ factorizes as $\tilde\pi(x, v) = \pi(x) \mu(v)$ with $\mu(v) \sim \exp(-v^\top v/2)$ a standard Gaussian.
Using this, we get that
\begin{align*}
\pi(x) Q(x, x')
&= \pi(x) \iint\limits_{v, v' \in \R^d} T((x, v), (x', v'))
    A((x, v), (x', v'))\, \mu(v) \dif v \dif v' \\
&= \pi(x) \iint\limits_{v, v' \in \R^d} T((x, v), (x', v'))
    \min \left\{ 1, \frac{\pi(x')\mu(v')}{\pi(x)\mu(v)} \right\}\, \mu(v) \dif v \dif v' \\
&= \iint\limits_{v, v' \in \R^d} T((x, v), (x', v'))
    \min \left\{ \pi(x) \mu(v), \pi(x')\mu(v') \right\} \dif v \dif v'.
\end{align*}
Since each term in the last expression is symmetric under the exchange of $(x, v)$ with $(x', v')$, we conclude that it is equal to $\pi(x') Q(x,x')$ for all $x, x'$, and conclude that the chain is reversible.
\end{proof}

\begin{proof}[Proof of \cref{lem:reversible}, part 2]
First recall that $(x',v') = \LF(x,v,t/\delta,\delta)$.
From \cref{sec:leapfrog} we know that the leapfrog integrator preserves the modified Hamiltonian and therefore we have
\[
\hat \cH(x,v) = \hat \cH(x',v') = \hat \cH(x',-v'). 
\]
Moreover, by \cref{eq: distance H and hat H} we have 
\[
\cH(x,v) - \hat \cH(x,v) = \frac{\delta^2}{8} \sum_{i \in [d]} \omega_i^4 x_i^2
\]
for all $x,v \in \R^d$. Combining these two identities we find that
\begin{align*}
\cH(x, v) - \cH(x',- v')
&= \left(\hat\cH(x, v) +\frac{\delta^2}{8} \sum_{i \in [d]} \omega_i^4 x_i^2\right) - \left(\hat\cH(x',- v') +\frac{\delta^2}{8} \sum_{i \in [d]} \omega_i^4 {x_i'}^2\right) \\
&= \frac{\delta^2}{8} \sum_{i \in [d]} \omega_i^4 (x_i^2 - {x_i'}^2),
\end{align*}
and hence the acceptance probability takes the form $A(x,x')$ as claimed.

From this, it easily follows that $Q_x$ takes the form $Q_x(x') = \hat Q_x(x') A(x,x')$ for $x \neq x'$:
\begin{align*}
Q_x(x')
&= \iint\limits_{v, v' \in \R^d} T((x, v), (x', v'))
    A((x, v), (x', v'))\, \mu(v) \dif v \dif v' \\
&= A(x, x') \iint\limits_{v, v' \in \R^d} T((x, v), (x', v'))\, \mu(v) \dif v \dif v'
= A(x, x') \hat Q_x(x'). \qedhere
\end{align*}
\end{proof}

\subsection{Concentration bounds on high-dimensional Gaussian random variables}
\label{sec:concentration}

Here we use concentration bounds on high-dimensional Gaussians to show that if $x \sim \pi$ or $x \sim \hat\pi$ then with high probability the quantity $\sum_{i \in [d]} \omega_i^4 x_i^2$ is close to $\sum_{i \in [d]} \omega_i^2$. We moreover show that in that case $\pi(x)$ and $\hat \pi(x)$ differ by at most a small multiplicative factor. 

We will use the following version of the Hanson-Wright inequality~\cite{Hanson1971BoundTailProbabilities} which gives a concentration inequality for quadratic forms of independent Gaussian random variables.  
\begin{theorem}[{Hanson-Wright inequality~\cite[Thrm~6.2.1]{Vershynin2018HighDimensionalProbabilityIntroduction}}] \label{thrm:HW}
Let $X = (X_1,\ldots,X_d) \in \R^d$ be a random vector with independent $\cN(0,1)$ coordinates. Let $A$ be a $d \times d$ matrix. Then, for every $t \geq 0$, we have 
\[
\P\Big[|X^\top A X - \Exp[X^\top A X]| \geq t \Big] \leq 2 \exp\left(-C \min\bigg\{\frac{t^2}{K^4\|A\|_F^2},\frac{t}{K^2 \|A\|}\bigg\}\right),
\]
where $K,C>0$ are constants.\footnote{The theorem holds more generally for independent mean zero \emph{sub-gaussian} variables $X_i$. The constant $K$ then upper bounds the \emph{sub-gaussian norm} of all $X_i$.} 
\end{theorem}
Note that if $X \in \R^d$ is a random vector with independent $\cN(0,1)$ coordinates, then so is $Y = U X$ for a rotation matrix $U$. This rotation-invariance allows us to again assume, for ease of notation, that the input precision matrix $B = \diag(\bm \omega)$. For convenience, recall that $\pi(x)
= \frac{\prod_i \omega_i}{(2\pi)^{d/2}} \exp\left(- \frac{1}{2} \sum_i x_i^2 \omega_i^2\right)$, and (cf.~\cref{eq:hat omega}) that $\hat \pi$ is constructed similarly using $\bm{\hat \omega}$ which is defined, for each $i \in [d]$, as $\hat \omega_i =   \omega_i \sqrt{ 1 - \frac{\delta^2 \omega_i^2}{4}}$.
We have $\omega_i^2-\hat \omega_i^2 = \frac{1}{4}\delta^2 \omega_i^4$. For $\gamma \geq 1$, we define the measurable set 
\begin{equation} \label{eq:def E}
E_\gamma := \left\{x \in \R^d \mid \Big|x^\top \diag(\bm \omega)^4 x - \sum_i \omega_i^2\Big| \leq \gamma \sqrt{\sum_{i \in [d]}\omega_i^4}\right\}. \end{equation}
The Hanson-Wright inequality gives us the following concentration of measure for $\pi$ and $\hat \pi$.

\begin{lemma} \label{lem:concentration Egamma}
Let $\gamma \geq 1$ and consider $E_\gamma$ as in \cref{eq:def E} then we have the following:
\begin{enumerate}
    \item Let $\pi(x) \propto \exp(-\frac{1}{2}x^\top \diag(\bm \omega)^2 x)$, then $\pi(E_\gamma) \geq 1- 2 \exp\big(-C\gamma\big)$ where $C>0$ is a constant. 
    \item If $0<\delta \leq \beta^{-1/2} d^{-1/4}$, then for $\hat \pi(x) \propto \exp(-\frac{1}{2} x^\top \diag(\bm{\hat \omega})^2 x)$ we have $\hat \pi(E_\gamma) \geq 1- 2 \exp(-C'\gamma)$ where $C'>0$ is a constant. 
\end{enumerate}
\end{lemma}
\begin{proof} We first prove the concentration of measure for $\pi$. 
We have 
\begin{align*}
\pi(E_\gamma) &= \P_{x \sim \pi}\left[\Big|x^\top \diag(\bm \omega)^4 x - \sum_i \omega_i^2\Big| \leq \gamma \sqrt{\sum_{i \in [d]}\omega_i^4}\right] \\
&= \P_{z \sim \mathcal N(0,I_d)}\left[\Big|z^\top \diag(\bm \omega)^2 z - \sum_i \omega_i^2\Big| \leq \gamma \sqrt{\sum_{i \in [d]}\omega_i^4}\right] 
\end{align*}
where we set $z_i = \omega_i x_i$ for each $i \in [d]$ and observe that $z_i \sim \mathcal N(0,1)$. We apply \cref{thrm:HW} to the vector $z$, matrix $A=\diag(\bm \omega)^2$, $t = \gamma \norm{A}_F$, and note that $\norm{A}_F \geq \norm{A}$ implies the lower bound
\[
    \min \left\{ \frac{(\gamma \norm{A}_F)^2}{K^4 \norm{A}_F^2}, \frac{\gamma \norm{A}_F}{K^2 \norm{A}} \right\} \geq \min \left\{ \frac{\gamma^2}{K^4}, \frac{\gamma}{K^2} \right\} \geq \gamma \min \{K^{-2}, K^{-4}\}.
\]
Therefore, for $C \leq \min \{K^{-2}, K^{-4}\}$ we obtain the desired bound for $\pi$.

We now use the same proof strategy to show concentration for $\hat \pi$. We have 
\begin{align*}
    \hat \pi(E_\gamma) &= \P_{x \sim \hat \pi}\left[\Big|x^\top \diag(\bm \omega)^4 x - \sum_i \omega_i^2\Big| \leq \gamma \sqrt{\sum_{i \in [d]}\omega_i^4}\right] \\
&= \P_{z \sim \mathcal N(0,I_d)}\left[\Big|z^\top \diag(\bm \omega)^4 \diag(\bm{\hat \omega})^{-2} z - \sum_i \omega_i^2\Big| \leq \gamma \sqrt{\sum_{i \in [d]}\omega_i^4}\right] \\
&\geq \P_{z \sim \mathcal N(0,I_d)}\left[\Big|z^\top \diag(\bm \omega)^4 \diag(\bm{\hat \omega})^{-2} z - \sum_i \omega_i^4/\hat \omega_i^2\Big| \leq \gamma \sqrt{\sum_{i \in [d]}\omega_i^4} - \Big|\sum_i \omega_i^2-\omega_i^4/\hat \omega_i^2\Big|\right] 
\end{align*}
By definition $\omega_i^4/\hat \omega_i^2 = \omega_i^2/(1-\delta^2 \omega_i^2/4)$, and the upper bound on $\delta$ implies that $\delta^2 \omega_i^2 \leq 2$. Using this bound, we get
\begin{align*}
    \left|\sum_i \omega_i^2 - \omega_i^4/\hat \omega_i^2\right|
    &= \sum_i \omega_i^2\left(1-\frac{1}{1-\delta^2\omega_i^2/4}\right) \\
    &\leq \sum_i \omega_i^2(1-1+\delta^2\omega_i^2/2)\\
    &= \frac12\sum_i \delta^2 \omega_i^4 \leq \frac{1}{2\sqrt{d}}\sum_i \omega_i^2 \leq \frac12 \sqrt{\sum_i \omega_i^4}.
\end{align*}
Again using the fact that $\omega_i^4 / \hat \omega_i^2 \leq 2 \omega_i^2$, we can further lower bound $\hat \pi(E_\gamma)$ as follows:
\begin{align*}
    \hat \pi(E_\gamma) &\geq \P_{z \sim \mathcal N(0,I_d)}\left[\Big|z^\top \diag(\bm \omega)^4 \diag(\bm{\hat \omega})^{-2} z - \sum_i \omega_i^4/\hat \omega_i^2\Big| \leq \frac{\gamma}4 \sqrt{\sum_{i \in [d]}(\omega_i^4/\hat \omega_i^2)^2}\right].
\end{align*}
We can then again apply \cref{thrm:HW} to obtain $\hat \pi(E_\gamma) \geq 1- 2 \exp(-C' \gamma)$ for a suitable constant $C'>0$. 
\end{proof}

Next we give a bound on $\hat \pi(x)/\pi(x)$ for all $x \in E_\gamma$, which we will use later to show that $\hat \pi$ can be used as a warm start for $\pi$. 

\begin{lemma} \label{lem: pointwise ratio in E}
Let $\pi(x) \propto \exp(-\frac{1}{2}x^\top \diag(\bm \omega)^2 x)$, let $\gamma \geq 1$ and consider $E_\gamma$ as defined in \cref{eq:def E}. Let~$\delta = \frac{1}{10\sqrt{\gamma \beta} d^{1/4}}$, set $\hat \omega_i =   \omega_i \sqrt{ 1 - \frac{\delta^2 \omega_i^2}{4}}$ for each $i \in [d]$, and let $\hat\pi(x) \propto \exp(-\frac{1}{2} x^\top \diag(\bm \hat\omega)^2 x)$. Then for all $x \in E_\gamma$ we have 
\[
0.9 \leq \frac{\hat \pi(x)}{\pi(x)} \leq 1.1.
\]
\end{lemma}
\begin{proof}
For $x \in \R^d$ we have 
\[
\frac{\hat\pi(x)}{\pi(x)}
= \left( \prod_i \Big(1 - \frac{\delta^2\omega_i^2}{4}\Big) \right)^{1/2} \exp\left(\frac{\delta^2}{8} \sum_i x_i^2 \omega_i^4 \right).
\]
We first obtain an upper bound on $\frac{\hat \pi(x)}{\pi(x)}$ for $x \in E_\gamma$. 
Using the inequality $1-z \leq \exp(-z)$ (which holds for all $z \in \R$), we obtain 
\[
\frac{\hat\pi(x)}{\pi(x)} \leq \exp\left(\frac{\delta^2}{8} \left(\sum_i x_i^2 \omega_i^4 -\omega_i^2\right)\right) \leq \exp\left(\frac{1}{8}\delta^2 \gamma \sqrt{\sum_{i \in [d]} \omega_i^4}\right) \leq  \exp\left(\frac{1}{800}\right) \leq 1.1
\]
where in the second inequality we use that $x \in E_\gamma$. 

We can similarly bound $\frac{\hat \pi(x)}{\pi(x)}$ from below for $x \in E_\gamma$. For this we use the inequality $1-z \geq \exp(-\eta z)$ which holds for $0 \leq z<1$ and $\eta \geq \frac{1}{z} \ln(\frac{1}{1-z})$. For $z\leq 1/2$ one has $\frac{1}{z} \ln(\frac{1}{1-z}) \leq 1+z$ and thus $\eta = 1+z$ suffices. We apply this with $z = \frac{\delta^2 \omega_i^2}{4} \leq \frac{1}{400 \gamma \sqrt{d}} <1/2$. This allows us to lower bound $\frac{\hat \pi(x)}{\pi(x)}$ as 
\begin{align*}
    \frac{\hat \pi(x)}{\pi(x)} &\geq \exp\left( - \frac12 \left(1+\frac{1}{400\gamma\sqrt{d}}\right)\frac{\delta^2}{4} \sum_i\omega_i^2 \right)\exp\left(\frac{\delta^2}{8} \sum_i x_i^2 \omega_i^4 \right) \\
    &\geq \exp\left( - \frac12 \frac{1}{400\gamma\sqrt{d}}\frac{\delta^2}{4} \sum_i\omega_i^2  - \frac{\delta^2}{8} \left|\sum_i x_i^2 \omega_i^4 - \sum_i \omega_i^2\right| \right) \\
    &\geq \exp\left( - \frac{1}{3200\cdot 100\gamma^2d \beta}\sum_i\omega_i^2  - \frac{1}{800} \right) \geq \exp\left( - \frac{1}{400} \right) \geq 0.9
\end{align*}
where in the third inequality we use that $\delta^2 = \frac{1}{100 \gamma \beta \sqrt{d}}$ and $x \in E_\gamma$. 
\end{proof}

Finally, we note that the acceptance probability is large on $E_\gamma$.

\begin{lemma} \label{lem:accept-prob}
Let $A(x,x')$ be the acceptance probability of the adjusted leapfrog HMC with step size~$\delta$.
If $x,x' \in E_\gamma$ then $A(x,x') \geq \exp\left( -\frac{\delta^2 \gamma}{4} d^{1/2} \beta \right)$. 
\end{lemma}
\begin{proof}
If both $x, x' \in E_\gamma$ then we have that $\sum_{i \in [d]} \omega_i^4 (x_i^2 - {x_i'}^2) \leq 2 \gamma \sqrt{\sum_i \omega_i^4} \leq 2 \gamma d^{1/2} \beta$. 
\end{proof}

\cref{lem: pointwise ratio in E,lem:accept-prob} tell us that the stepsize $\delta$ should scale with $\gamma,d$ and $\beta$ as 
\begin{equation} \label{eq:choice-delta}
\delta = \frac{1}{10\sqrt{\gamma \beta}d^{1/4}}.
\end{equation}
This choice of $\delta$ ensures a high acceptance probability whenever $x,x' \in E_\gamma$ and a pointwise bound on the ratio $\hat \pi(x)/\pi(x)$ for $x \in E_\gamma$. In the next section we tune the choice of $\gamma \geq 1$ to apply an argument based on the $s$-conductance. 

\subsection{\texorpdfstring{$s$}{s}-conductance and warm start}
We will bound the mixing time of the Metropolis-adjusted chain using the so-called \emph{$s$-conductance}.
This is a generalization of the conductance that allows to ignore small subsets of measure $\pi(S) \leq s$.

\begin{definition}[$s$-conductance]
Let $0<s<1/2$ and define the $s$-conductance~$C_s$ of a Markov chain with transition kernel $T$ and stationary distribution $\pi$ as
\[
C_s \coloneqq \inf\left\{ C_s(S) \mid S \subseteq \R^d \text{ measurable}, \ s < \pi(S) \leq \frac{1}{2}\right\},
\; \text{ with } \;
C_s(S) \coloneqq \frac{\int_S T(x,S^c) \pi(\dx)}{\pi(S)-s}.
\]
\end{definition}

The $s$-conductance leads to a mixing time bound through the following theorem from Lovász and Simonovits \cite{Lovasz1993RandomWalksConvex} (the exact formulation below is from \cite[Lem.~1]{Wu2021MinimaxMixingTime}).
It uses a \emph{warmness} parameter~$D^{\mu_0,\pi}_s$ between the initial distribution $\mu_0$ and target distribution $\pi$, which for $0<s<1/2$ is defined by
\[
D^{\mu_0,\pi}_s
\coloneqq \sup\{|\mu_0(A)-\pi(A)| \,:\, A \subseteq \R^d \text{ measurable, } \pi(A) \leq s\}. 
\]

\begin{lemma}[\cite{Lovasz1993RandomWalksConvex}] \label{lem:s-cond-bound}
Consider a reversible, lazy\footnote{A lazy chain takes a step with probability $1/2$, and otherwise does nothing.} Markov chain with transition kernel $R$, stationary distribution $\pi$ and initial distribution $\mu_0$.
Then for any $K \geq 0$ it holds that
\[
\| R^K_{\mu_0} - \pi \|_{\TV}
\leq D^{\mu_0,\pi}_s + \frac{D^{\mu_0,\pi}_s}{s}
    \left(1 - \frac{C_s^2}{2}\right)^K.
\]
\end{lemma}

Using \cref{lem: pointwise ratio in E} we can prove that the stationary distribution $\hat\pi$ of the \emph{unadjusted} chain $\hat Q$ for sufficiently small step size forms a warm start, if we take $\gamma \in \Theta(\log(1/s))$.

\begin{lemma}[unadjusted warm start] \label{lem:s-cond-warm-start}
Let $\pi(x) \propto e^{-\frac{1}{2}x^\top \diag(\bm \omega)^2 x}$ and let $\hat\pi(x) \propto e^{-\frac{1}{2} x^\top \diag(\bm{\hat \omega})^2 x}$ with $\hat \omega_i =   \omega_i \sqrt{ 1 - \frac{\delta^2 \omega_i^2}{4}}$.
For any $0 < s < 1/2$, if $\delta \leq \frac{C}{\sqrt{\beta \log(1/s)}d^{1/4}}$ for a sufficiently small constant $C>0$, then 
\[
D^{\hat\pi,\pi}_s
\leq 3s.
\]
\end{lemma}
\begin{proof}
Consider the set $E_\gamma$ defined in \eqref{eq:def E} for a sufficiently large $\gamma \in O(\log(1/s))$.
Then by \cref{lem:concentration Egamma,lem: pointwise ratio in E} both $\pi(E_\gamma) \geq 1-s$ and $\hat \pi(E_\gamma) \geq 1-s$, 
and $\hat\pi(x)/\pi(x) \leq 1.1$ for all $x \in E_\gamma$. Now let $A \subseteq \R^d$ with $\pi(A) \leq s$. Then we have
\begin{align*}
|\hat\pi(A)-\pi(A)|
&= |\hat\pi(A \cap E_\gamma) + \hat\pi(A \cap E_\gamma^c)
    - \pi(A \cap E_\gamma) - \pi(A \cap E_\gamma^c)| \\
&\leq |\hat\pi(A \cap E_\gamma) - \pi(A \cap E_\gamma)|
    + \hat\pi(A \cap E_\gamma^c) + \pi(A \cap E_\gamma^c) \\
&\leq \pi(A \cap E_\gamma)+ \hat\pi(A \cap E_\gamma^c)
    + \pi(A \cap E_\gamma^c) \\
&\leq \pi(A) + s + s \leq 3s.
\end{align*}
Here in the second inequality we use that $|\hat\pi(x)-\pi(x)| \leq \pi(x)$ for all $x \in E_\gamma$. 
\end{proof}

\subsection{Bounding the \texorpdfstring{$s$}{s}-conductance of the adjusted HMC chain}

To bound the $s$-conductance of the adjusted chain, we first bound the $s$-conductance of the \emph{unadjusted} HMC chain $\hat Q$, and then relate both conductances.
For the unadjusted chain, we can use our bounds on the mixing time of that chain to lower bound its conductance.

\begin{lemma}[$s$-conductance unadjusted HMC] \label{lem: conductance of Q hat}
Let $0 < s < 1/2$ and let $\hat{C}_s$ be the $s$-conductance of the unadjusted HMC chain $\hat Q$ with step size $\delta \leq \frac{C}{\sqrt{\beta \log(1/s)}d^{1/4}}$ for a sufficiently small constant $C>0$.
Then
\[
\hat C_s \in \Omega(1/\log(d \kappa \log(1/s))).
\]
\end{lemma}
\begin{proof}
First consider the $s$-conductance $\hat C_s^{(K)}$ of the \emph{$K$-step} kernel $\hat Q^{K}$.
From \cref{lem:mixing-unadjusted} we know that $\| \hat Q^{K}_x - \hat\pi \|_{\TV} \leq 1/10$ for $K \geq C \log(d\kappa(\sqrt{\alpha}\norm{x}_\infty + 1))$ for an appropriate constant $C>0$.
In particular, if $x \in E_\gamma$ with $\gamma \geq 1$ then $\|x\|_\infty \leq \sqrt{\frac{(\gamma+1) \kappa d}{\alpha}}$ and hence $\| \hat Q^{K}_x - \hat\pi \|_{\TV} \leq 1/10$ for all $x \in E_\gamma$ and $K \geq C'\log(\gamma d \kappa)$ for an appropriate constant $C'>0$.
By \cref{lem:concentration Egamma} we can ensure $\hat\pi(E_\gamma) \geq 1-s$ by picking $\gamma \in O(\log(1/s))$ (recall that $\delta = \frac{1}{10 \sqrt{\gamma \beta} d^{1/4}}$). This choice of $\gamma$ ensures there exists a $K \in O(\log(d \kappa \log(1/s)))$ with the above properties.
Combining these properties, for any $S$  for which $s < \hat\pi(S) \leq 1/2$ we have that
\begin{align*}
\hat C_s^{(K)}(S)
= \frac{\int_S \hat\pi(x) \hat Q_x^{K}(S^c)}{\hat\pi(S) - s}
&\geq \frac{\int_{S \cap E_\gamma} \hat\pi(x) \hat Q_x^{K}(S^c)}{\hat\pi(S) - s} \\
&\geq \frac{\hat\pi(S \cap E_\gamma) (\hat\pi(S^c) - 1/10)}{\hat\pi(S) - s}
\geq \hat\pi(S^c) - \frac{1}{10} \geq \frac{2}{5},
\end{align*}
and hence $\hat C_s^{(K)} \geq 2/5$.

Now we can use the fact that $\hat C_s^{(K)} \leq K \hat C_s^{(1)} = K\hat C_s$ to conclude that $\hat C_s \geq 2/(5K)$, which is $\Omega(1/\log(d\kappa\log(1/s)))$ as claimed.
To see that $\hat C_s^{(K)} \leq K \hat C_s^{(1)}$ (which is well-known, see e.g.~\cite[Eq.~(7.10)]{Levin2017MarkovChainsMixing}), define $\hat\pi_S$ by $\hat\pi_S(x) = \hat\pi(x)$ for $x \in S$ and $\hat\pi_S(x) = 0$ elsewhere.
Then note that $\hat C_s^{(K)}(S) = \| Q^{K}_{\hat\pi_S} - \hat\pi_S \|_{\TV}/(\hat\pi(S)-s)$.
Using a telescoping sum and a triangle inequality we can bound 
\begin{align*}
\| Q^{K}_{\hat\pi_S} - \hat\pi_S \|_{\TV}
&\leq \| Q^{K}_{\hat\pi_S} - Q^{K-1}_{\hat\pi_S} \|_{\TV}
    + \| Q^{K-1}_{\hat\pi_S} - Q^{K-2}_{\hat\pi_S} \|_{\TV}
    + \dots + \| Q_{\hat\pi_S} - \hat\pi_S \|_{\TV} \\
&\leq K \| Q_{\hat\pi_S} - \hat\pi_S \|_{\TV},
\end{align*}
where the second inequality follows from submultiplicativity of the total variation distance.
Dividing both sides by $\hat\pi(S)-s$ and taking the infimum over $S$ proves that $\hat C_s^{(K)} \leq K \hat C_s^{(1)}$.
\end{proof}

To relate the $s$-conductance of the adjusted chain to the one of the unadjusted chain, we use the properties of $\pi$ and $\hat \pi$ shown in \cref{sec:concentration}: there is a set $E \subseteq \R^d$ of large measure on which $\pi$ and $\hat \pi$ pointwise differ by at most a small multiplicative constant. Moreover, if both $x \in E$ and $x' \in E$, then the acceptance probability of the adjusted chain satisfies $A(x,x') \geq 99/100$.

\begin{lemma}[$s$-conductance adjusted HMC] \label{lem:bound-Cs}
Let $0 < s < C/\log(d \kappa)$ for a sufficiently small constant $C>0$, and let $C_s$ and $\hat C_{s/2}$ be the $s$-conductance and the $s/2$-conductance of the adjusted and unadjusted chains $Q$ and $\hat Q$ with step size $\delta \leq \frac{C'}{ \sqrt{\beta \log(1/s) } d^{1/4}}$ for a sufficiently small constant $C'>0$. 
Then
\[
    C_s \geq \hat C_{s/2}/2.
\]
\end{lemma}
\begin{proof}
Our goal is to lower bound $\frac{1}{\pi(S) - s} \int_S \pi(x) Q(x, S^c) \dif x$ for all sets $S$ such that $s < \pi(S) \leq \frac12$. To this end, we will use that by \cref{lem:concentration Egamma,lem: pointwise ratio in E,lem:accept-prob} the set $E \coloneqq E_\gamma \subset \R^d$ (defined in \cref{eq:def E}) for a suitable $\gamma \in \Theta(\log(1/s))$ and $\delta = \frac{1}{10\sqrt{\gamma \beta}d^{1/4}}$ (as in \cref{eq:choice-delta}) satisfies
\begin{enumerate}
    \item $\pi(E^c) \leq s/10$,
    \item $\hat\pi(E^c) \leq s^2/10$,
    \item $0.9 \leq \frac{\hat\pi(x)}{\pi(x)} \leq 1.1$ for all $x \in E$,
    \item the acceptance probability $A(x,x') \geq 99/100$ for all $x,y \in E$.
\end{enumerate}
Note that in \cref{lem: conductance of Q hat} we have shown that $\hat C_{s/2} \in \Omega(1/\log(d \kappa \log(1/s)))$. Therefore, for $s < C/\log(d\kappa)$ for a small enough constant $C>0$, we have $s \leq \hat C_{s/2}$ and thus $\hat \pi(E^c) \leq s \hat C_{s/2}/10$. 

We can use this to lower bound the integral
\begin{align*}
    \int_S \pi(x) Q(x, S^c) \dif x &\geq \int_{S \cap E} \pi(x) Q(x, S^c \cap E) \dif x \\
    &= \int_{S\cap E} \pi(x) \int_{S^c \cap E} Q(x, y) \dif y \dif x \\
    &= \int_{S\cap E} \pi(x) \int_{S^c \cap E} \hat Q(x, y) A(x, y) \dif y \dif x \\
    &\geq 0.85 \int_{S\cap E} \hat\pi(x) \int_{S^c \cap E} \hat Q(x, y) \dif y \dif x \\
    &= 0.85 \int_{S\cap E} \hat\pi(x) \hat Q(x, S^c \cap E) \dif x \\
    &= 0.85 \left( \int_{S\cap E} \hat\pi(x) \hat Q(x, S^c \cup E^c) \dif x - \int_{S\cap E} \hat\pi(x) \hat Q(x, E^c) \dif x \right) \\
    &\geq 0.85 \left( \int_{S\cap E} \hat\pi(x) \hat Q(x, S^c \cup E^c) \dif x - \hat \pi(E^c) \right),
\end{align*}
where the last inequality follows from detailed balance:
\[
\int_{S\cap E} \hat\pi(x) \hat Q(x, E^c) \dif x
= \int_{E^c} \hat\pi(x) \hat Q(x, S \cap E) \dif x
\leq \hat\pi(E^c).
\]
We recognize the last integral as the ergodic flow from the set $S':= S \cap E$ to its complement, and so we can lower bound it in terms of the conductance of $\hat Q$, provided that $S'$ has an appropriate measure according to $\hat \pi$. We bound $\hat\pi(S')$ from below
\[
    \hat\pi(S') \geq 0.9\pi(S') = 0.9 (\pi(S) - \pi(S \cap E^c)) \geq 0.9s - \pi(E^c) \geq 0.8s,
\]
and from above:
\[
    \hat\pi(S') \leq 1.1 \pi(S') \leq 1.1\pi(S) \leq 0.55.
\]
We proceed in two different ways depending on the measure $\hat\pi(S')$.
\begin{enumerate}
    \item If $0.8s \leq \hat\pi(S') \leq 1/2$, we have the lower bound
    \begin{align*}
        C_s = \frac{\int_S \pi(x) Q(x, S^c) \dif x}{\pi(S) - s} &\geq 0.85 \frac{\hat C_{s/2}(\hat \pi(S') - s/2) - \hat\pi(E^c)}{\pi(S) - s} \\
        &\geq 0.85 \frac{\hat C_{s/2}(\hat \pi(S') - 0.6s)}{\pi(S) - s} \\
        &\geq 0.85 \frac{\hat C_{s/2}(0.9 \pi(S') - 0.6s)}{\pi(S) - s} \\
        &\geq 0.85 \frac{\hat C_{s/2}(0.9 \pi(S) - \pi(E^c) - 0.6s)}{\pi(S) - s} \\
        &\geq 0.85 \frac{\hat C_{s/2}(0.9 \pi(S) - 0.7s)}{\pi(S) - s} \\
        &\geq 0.85 \frac{\hat C_{s/2}(0.7 \pi(S) - 0.7s)}{\pi(S) - s} \geq \frac{\hat C_{s/2}}{2}.
    \end{align*}
    \item If $1/2 \leq \hat \pi(S') \leq 0.55$, we have $s \leq \hat\pi(S'^c) \leq 1/2$. Additionally, we know that $\hat Q$ satisfies detailed balance:
    \[
        \int_{S'} \hat\pi(x) \hat Q(x, S'^c) \dif x = \int_{S'^c} \hat \pi(x) \hat Q(x, S') \dif x.
    \]
    Therefore, we have the following lower bound
    \begin{align*}
        C_s = \frac{\int_S \pi(x) Q(x, S^c) \dif x}{\pi(S) - s} &\geq 0.85 \frac{\hat C_{s/2}(\hat \pi(S'^c) - s/2) - \hat\pi(E^c)}{\pi(S^c) - s} \\
        &\geq 0.85 \frac{\hat C_{s/2}(\hat \pi(S'^c) - 0.6s)}{\pi(S^c) - s} \\
        &= 0.85 \frac{\hat C_{s/2}(1 - \hat \pi(S') - 0.6s)}{1 -\pi(S) - s} \\
        &\geq 0.85 \frac{\hat C_{s/2}(1 - 1.1\pi(S) - 0.6s)}{1 -\pi(S) - s} \\
        &\geq 0.85 \frac{\hat C_{s/2}(1 - 1.1\pi(S) - 0.6s)}{1 -\pi(S) - 0.6s} \geq \frac{\hat C_{s/2}}{2}.
        \qedhere
    \end{align*}
\end{enumerate}
\end{proof}

\subsection{Mixing time of adjusted HMC}

We can now plug our bounds on the $s$-conductance into \cref{lem:s-cond-bound} to get the following bound on the mixing time of the (lazy) Metropolis-adjusted HMC chain,\footnote{Making the chain lazy reduces the $s$-conductance only by a factor 2.} when starting from a warm start.

\begin{theorem}[Metropolis-adjusted HMC with warm start] \label{thm:adjusted-warm-mixing-time}
Let $0 < \eps < C/\log(d \kappa)$ for a sufficiently small constant $C>0$, and let $\mu_0$ be an initial distribution with warmness $D^{\mu_0,\pi}_s \leq \eps/2$ for $s = \eps/6$.
There exist constants $C',C'' > 0$ such that for every $x \in \R^d$, if
\[
K
\geq C' \log(d \kappa \log(1/\eps)) \log(1/\eps)
\quad \text{ and } \quad
\delta \leq  \frac{C''}{\sqrt{\beta \log(1/\eps)} d^{1/4}},
\]
then
\[
\| Q^K_{\mu_0} - \pi \|_{\TV}
\leq \eps
\]
where $\pi \propto \exp(-x^\top B x/2)$ and $Q$ is the kernel of the (lazy) Metropolis-adjusted leapfrog HMC chain with step size $\delta$. 
\end{theorem}
\begin{proof}
For $s = \eps/6$ and our choice of $\delta$ we know from \cref{lem:bound-Cs,lem: conductance of Q hat} that $Q$ has $s$-conductance $C_s \in \Omega(1/\log(d \kappa \log(1/s)))$.
By invoking \cref{lem:s-cond-bound} we know that
\[
\| Q^K_{\mu_0} - \pi \|_{\TV}
\leq D_s + \frac{D_s}{s} \left( 1 - \frac{C_s^2}{2} \right)^K
\leq \frac{\eps}{2} + 3 \left( 1 - \frac{C_s^2}{2} \right)^K
\leq \eps
\]
for $K \in \Omega(\log(1/\eps)/C_s)$ and hence $K \in \Omega(\log(d \kappa \log(1/\eps))\log(1/\eps))$.
\end{proof}

Hence, starting from a warm start $\mu_0$ we can sample from a distribution $\eps$-close to $\pi$ in TV-distance  using $\tO(\sqrt{\kappa} d^{1/4} \log(1/\eps))$ gradient evaluations.
To get around this warm start, recall from \cref{lem:s-cond-warm-start} that the stationary distribution of the \emph{unadjusted} chain (with sufficiently small step size $\delta$) provides a warm start for the adjusted chain. This gives the following, main theorem.

\begin{theorem}[Metropolis-adjusted HMC]
\label{thm:adjusted-mixing-time}
Let $0 < \eps < C/\log(d \kappa)$ for a sufficiently small constant $C>0$. There exists constants $C_0',C', C'' > 0$ such that for every $x \in \R^d$, if 
\[
K \geq C' \log(d\kappa \log(1/\eps))\log(1/\eps),\;\;
K_0 \geq C_0'\log\left(\frac{d\kappa(\sqrt{\alpha}\norm{x}_\infty + 1)}{\eps}\right),\;\;
\delta \leq  \frac{C''}{\sqrt{\beta \log(1/s)} d^{1/4}},
\]
then
\[
\| (Q^K \circ \hat Q^{K_0})_x - \pi \|_{\TV}
\leq \eps
\]
where $\pi \propto \exp(-x^\top B x/2)$ and $Q$ (resp.~$\hat Q$) is the kernel of the (lazy) Metropolis-adjusted (resp.~unadjusted) leapfrog HMC chain with step size $\delta$.
We can thus obtain a sample from a distribution that is $\eps$-close to $\pi$ in TV-distance using $\tO(\sqrt{\kappa} d^{1/4} \log(1/\eps))$ gradient evaluations.
\end{theorem}
\begin{proof}
From \cref{lem:s-cond-warm-start} we know that there exists a constant $C''>0$ such that if $\delta \leq \frac{C''}{\sqrt{\beta \log(1/s)} d^{1/4}}$, then $\hat\pi$ is such that $D^{\hat\pi,\pi}_s \leq \eps/4$ for $s = \eps/12$, i.e., $\hat \pi$ is warm for $\pi$. \cref{thm:adjusted-warm-mixing-time} shows that there exists a constant $C'>0$ such that for all $K
\geq C' \log(d \kappa \log(1/\eps)) \log(1/\eps)$ we have $\| Q_{\hat\pi}^K - \pi \|_{\TV} \leq \eps/2$. On the other hand, for the unadjusted chain, by \cref{cor:idealized-mixing}, there exists a constant $C_0'>0$ such that for all $x \in \R^d$ and $K_0 \geq C_0' \log\left(\frac{d\kappa(\sqrt{\alpha}\norm{x}_\infty + 1)}{\eps}\right)$ we have $\|\hat Q^{K_0}_x - \hat \pi \|_{\TV} \leq \eps/2$. Combining these two estimates we obtain for such $K$ and $K_0$ that
\begin{align*}
\| (Q^K \circ \hat Q^{K_0})_x - \pi \|_{\TV}
&\leq \| (Q^K \circ \hat Q^{K_0})_x - Q^K_{\hat\pi} \|_{\TV}
    + \| Q^K_{\hat\pi} - \pi \|_{\TV} \\
&\leq \| \hat Q^{K_0}_x - \hat\pi \|_{\TV}
    + \| Q^K_{\hat\pi} - \pi \|_{\TV}
\leq \eps,
\end{align*}
where we used submultiplicativity ($\| Q^K_\mu - Q^K_\nu \|_{\TV} \leq \| \mu - \nu \|_{\TV}$) in the second inequality.
\end{proof}

\section{Conclusions and open questions}

To conclude, we studied the Hamiltonian Monte Carlo algorithm for sampling from high-dimensional Gaussian distributions, focusing on the dependency on both condition number $\kappa$ and dimension $d$ of the Gaussian.
We showed that a HMC algorithm with the leapfrog integrator and long, randomized integration times can be used to sample from a distribution $\eps$-close to a Gaussian distribution by making only $\tO(\sqrt{\kappa} d^{1/4} \log(1/\eps))$ gradient queries.
This scaling seems optimal for leapfrog HMC in both the dimension and the condition number (by well-known scaling limits \cite{Duane1987HybridMonteCarlo,Neal2011MCMCUsingHamiltonian}).

The $\sqrt{\kappa}$-dependency also improves over similar, preceding work on leapfrog HMC that achieved at best a linear $\kappa$-dependency \cite{Mangoubi2018DimensionallyTightBounds,Chen2020FastMixingMetropolized}.
While these works typically consider more general logconcave distributions, we feel that our work enhances the possibility of obtaining a similar $\sqrt{\kappa}$-dependency for such distributions as well.
This would disprove the $\Omega(\kappa)$ versus $O(\sqrt{\kappa})$ gap that was suggested by Lee, Shen and Tian~\cite{Lee2020LogsmoothGradientConcentration} between logconcave sampling and convex optimization, respectively.

\printbibliography

\end{document}